\documentclass[letterpaper, 10 pt, conference]{ieeeconf}  

\IEEEoverridecommandlockouts                              
\usepackage{bm}
\usepackage{enumerate}
\usepackage{amsmath}
\usepackage{graphicx}
\usepackage{amsmath}

\usepackage{amsthm}
\usepackage{subcaption}
\usepackage{breqn}
\usepackage{cite}
\usepackage[table]{xcolor}
\usepackage{booktabs}
\usepackage{empheq}
\usepackage{amssymb}
\usepackage{hyperref}
\usepackage{xcolor}

\newtheorem*{theorem*}{Theorem}
\newtheorem{theorem}{Theorem}

\newtheorem{definition}{Definition}
\newtheorem{assumption}{Assumption}

\newtheorem{proposition}{Proposition}
\newtheorem{remark}{Remark}

\newcommand{\bmb}[1]{\bar{\bm{#1}}}

\newcommand{\bmd}[1]{\dot{\bm{#1}}} 
\newcommand{\bmdd}[1]{\ddot{\bm{#1}}}

\newcommand{\bracketmat}[2]{ \left[ \begin{array}{#1} #2 \end{array} \right] }

\title{\LARGE \bf
Grasp Constraint Satisfaction for Object Manipulation using Robotic Hands
}

\author{Wenceslao Shaw-Cortez, Denny Oetomo, Chris Manzie, and Peter Choong
\thanks{W. Shaw-Cortez, D. Oetomo, and C. Manzie are with the MIDAS Laboratory in the School of Electrical, Mechanical, and Infrastructure Engineering, 
	University of Melbourne, 3010, Australia,
        {\tt\small shaww@student.unimelb.edu.au, doetomo@unimelb.edu.au, manziec@unimelb.edu.au}.}
 \thanks{ P. Choong is with the Department of Surgery, St. Vincent's Hospital, 3065,
 		Australia,
 		{\tt\small  pchoong@unimelb.edu.au}.}
}

\begin{document}

\maketitle
\thispagestyle{empty}
\pagestyle{empty}

\begin{abstract}

For successful object manipulation with robotic hands, it is important to ensure that the object remains in the grasp at all times. In addition to grasp constraints associated with slipping and singular hand configurations, excessive rolling is an important grasp concern where the contact points roll off of the fingertip surface. Related literature focus only on a subset of grasp constraints, or assume grasp constraint satisfaction without providing guarantees of such a claim. In this paper, we propose a control approach that systematically handles all grasp constraints. The proposed controller ensures that the object does not slip, joints do not exceed joint angle constraints (e.g. reach singular configurations), and the contact points remain in the fingertip workspace. The proposed controller accepts a nominal manipulation control, and ensures the grasping constraints are satisfied to support the assumptions made in the literature. Simulation results validate the proposed approach.

\end{abstract}

\section{Introduction}

A primary objective of robotic hand research is to manipulate the environment to achieve a desired goal \cite{Bicchi2000}. This can be accomplished in a hierarchical grasp framework where a high-level planner plans the grasp, forms the grasp, and manipulates the object \cite{Hang2016, Lippiello2013a}. The focus of this work is in the manipulation aspect of the hierarchical approach, which consists of rotating/translating the object within the grasp, and hereafter is referred to as in-hand manipulation. In-hand manipulation controllers are used to track a desired object reference trajectory, while keeping the object inside the grasp \cite{Ozawa2017}. Thus for successful in-hand manipulation, it is paramount to guarantee that the object remains in the grasp during the manipulation motion.

A failed grasp can result from slipping, joint over-extension, and excessive rolling. Slipping is an obvious grasping concern, which has been formulated as a constraint satisfaction problem and extensively addressed in the literature \cite{ShawCortez2018b, Caldas2015, Nahon1992}. Joint over-extension relates to joints exceeding feasible joint angle constraints (e.g joint workspace, hardware capabilities, singular hand configurations), which inhibits the robotic hand from applying necessary contact forces to prevent grasp failure \cite{Murray1994}. Excessive rolling, as defined here, is when the contact points roll off of the fingertip surface. In-hand manipulation inherently relies on rolling motion for object manipulation \cite{Montana1988}. However excessive rolling motion may cause the contact points to leave the fingertip surface. The fingertip surface can refer to the the sensor surface or the entire surface of the fingertip such that excessive rolling compromises the ability to manipulate the object. For successful manipulation, the manipulation controller must guarantee grasp constraint satisfaction such that the contact points remain in the fingertip workspace, whilst simultaneously ensuring no slip and no over-extension.

To date, there exist an abundance of multi-fingered robotic hands, including tactile sensors, along with methods of controlling them \cite{Ozawa2017, Kappassov2015}.  However, many manipulation controllers focus on the object manipulation task and assume grasp constraint satisfaction, but provide no guarantees to support that assumption \cite{Kawamura2013, Wimbock2012}. Other methods focus on one grasp constraint, such as no slipping, but do not systematically consider all grasp constraints \cite{Caldas2015, Nahon1992}. Existing methods of addressing all grasping constraints are via motion planning approaches \cite{Hertkorn2013, Han2000, Cherif1999}, however they rely on quasi-static assumptions that generally do not hold in a dynamic manipulation setting. Furthermore those approaches require large computational resources and may not be conducive to real-time applications \cite{Hertkorn2013}. In the hierarchical grasp framework, it is difficult to determine a priori the region of attraction of the in-hand manipulation controller as it depends on both the object and initial grasp configuration. Therefore a given reference trajectory from the high-level planner may result in grasp constraint violation and thus grasp failure. There is no existing method to prevent excessive rolling, no slip, and over-extension for real-time object manipulation.

In this paper, we present a novel controller that systematically handles no slipping, over-extension, and excessive rolling constraints for object manipulation. The proposed controller admits a nominal manipulation control, as found in \cite{Ozawa2017}, and outputs a minimally close, in the 2-norm sense, manipulation control torque, while adhering to the grasping constraints. The effectiveness of the proposed controller is validated in simulation.

\subsection*{Notation}

Throughout this paper, an indexed vector $\bm{v}_i \in \mathbb{R}^p$ has an associated concatenated vector $\bm{v} \in \mathbb{R}^{pn}$, where the index $i$ is specifically used to index over the $n$ contact points in the grasp. The notation $\bm{v}^{\mathcal{\mathcal{F}}}$ indicates that the vector $\bm{v}$ is written with respect to a frame $\mathcal{F}$, and if there is no explicit frame defined, $\bm{v}$ is written with respect to the inertial frame, $\mathcal{P}$.  The operator $(\cdot)\times$ denotes the skew-symmetric matrix representation of the cross-product. The $k\times k$ identity matrix is denoted $I_{k\times k}$. The Moore-Penrose inverse of a matrix $B$ is denoted by $B^\dagger$.

\section{Background} \label{sec:system model}

\subsection{Hand-Object System}

Consider a fully-actuated, multi-fingered hand grasping a rigid, convex object at $n \in \mathbb{Z}_{>0}$ contact points.  Each finger consists of $m_i \in \mathbb{Z}_{>0}$ joints with smooth, convex fingertips of high stiffness. Let the finger joint configuration be described by the joint angles, $\bm{q}_i \in \mathbb{R}^{m_i}$. The full hand configuration is defined by the joint angle vector, $\bm{q} = (\bm{q}_1, \bm{q}_2, ..., \bm{q}_n)^T \in \mathbb{R}^m$, where $m = \sum_{i=1}^n m_i$ is the total number of joints. Let the inertial frame, $\mathcal{P}$, be fixed on the palm of the hand, and a  fingertip frame, $\mathcal{F}_i$, fixed at the point $\bm{p}_{f_i} \in \mathbb{R}^3$. The translational and rotational velocities of $\mathcal{F}_i$ with respect to $\mathcal{P}$ are denoted $\bm{v}_{f_i}, \bm{\omega}_{f_i} \in \mathbb{R}^3$, respectively. The contact frame, $\mathcal{C}_i$, is located at the contact point, $\bm{p}_{c_i} \in \mathbb{R}^3$.  A visual representation of the contact geometry for the $i$th finger is shown in Figure \ref{fig.contactpic}. 

\begin{figure}[hbtp]
\centering
\includegraphics[scale=0.3]{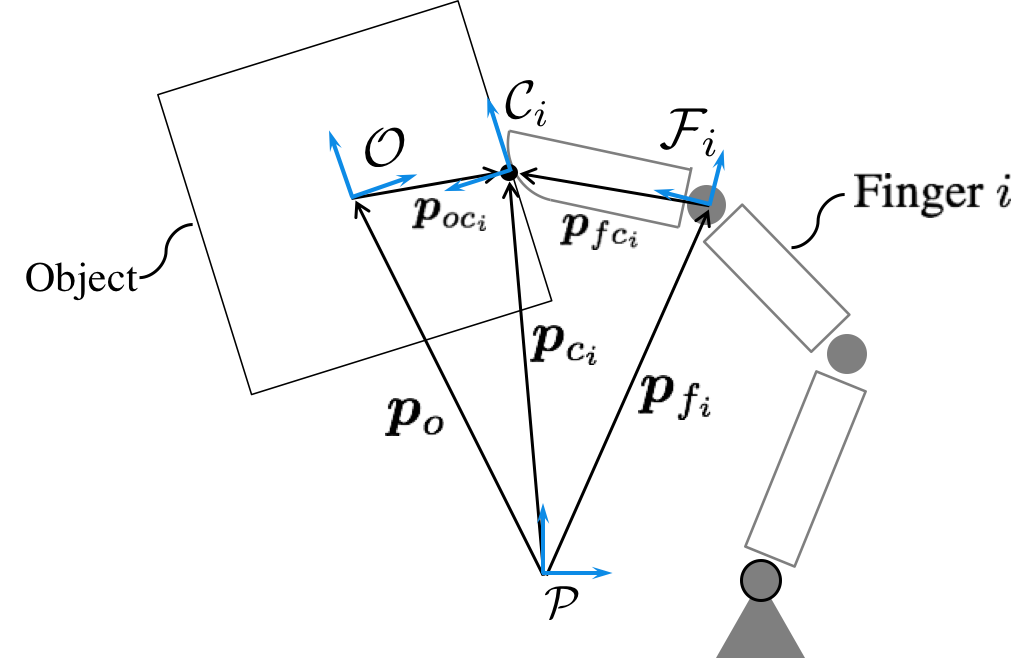}
\caption{A visual representation of the contact geometry for contact $i$. }  \label{fig.contactpic}
\end{figure}

Let $\mathcal{O}$ be a reference frame fixed at the object center of mass $\bm{p}_o \in \mathbb{R}^3$, and $R_{po} \in  SO(3)$ is the rotation matrix, which maps from $\mathcal{O}$ to $\mathcal{P}$. The respective inertial translation and rotational velocities of the object are $\bm{v}_o, \bm{\omega}_o \in \mathbb{R}^3$.  The object state is $\bm{x}_o \in \mathbb{R}^6$, with $\bmd{x}_o = (\bm{v}_o, \bm{\omega}_o)$. The position vector from $\mathcal{O}$ to the respective contact point is $\bm{p}_{oc_i} \in \mathbb{R}^3$.

The hand/object dynamics are respectively defined as \cite{Murray1994}:
\begin{equation} \label{eq:hand dynamics}
M_h \bmdd{q} + C_h \bmd{q} = -J_h^T \bm{f}_c +  \bm{\tau}_e + \bm{u}
\end{equation}
\begin{equation} \label{eq:object dynamics}
M_o\bmdd{x}_o + C_o \bmd{x}_o = G \bm{f}_c + \bm{w}_e
\end{equation}
where $M_h := M_h(\bm{q}) \in \mathbb{R}^{m\times m}, M_o := M_o(\bm{x}_o) \in \mathbb{R}^{6\times 6}$ are the respective hand and object inertia matrices, $C_h :=  C_h(\bm{q},\bmd{q}) \in \mathbb{R}^{m\times m}, C_o := C_o(\bm{x}_o,\bmd{x}_o) \in \mathbb{R}^{6\times 6}$ are the respective hand and object Coriolis and centrifugal matrices, $\bm{\tau}_e :=  \bm{\tau}_e(t) \in \mathbb{R}^m$ is the disturbance torque acting on the joints,  $\bm{w}_e :=  \bm{w}_e(t) \in \mathbb{R}^6$ is an external wrench disturbing the object, and  $\bm{u} \in \mathbb{R}^m$ is the joint torque control input for a fully-actuated hand. The grasp map, $G :=  G(\bm{p}_{oc}) \in \mathbb{R}^{6 \times 3n}$ maps the contact force, $\bm{f}_c$, to the net wrench acting on the object. The hand Jacobian, $J_h := J_h(\bm{q},\bm{p}_{fc}) \in \mathbb{R}^{3n \times m}$,  relates the motion of the hand and velocity of the contact points. The hand Jacobian is a block diagonal matrix of the individual hand Jacobian matrices:
\begin{equation}
J_{h_i}(\bm{q}_i,\bm{p}_{fc_i}) = \bracketmat{cc}{I_{3\times 3} & -(\bm{p}_{fc_i})\times} J_{s_i}(\bm{q}_i)
\end{equation}
 where $J_{s_i}(\bm{q}_i) \in \mathbb{R}^{6 \times mi}$ is the spatial manipulator Jacobian that maps $\bmd{q}_i \mapsto (\bm{v}_{f_i}, \bm{\omega}_{f_i})$ \cite{Murray1994}.

 When the contact points do not slip, the following grasp constraint holds \cite{Cole1989}:
\begin{equation}\label{eq:grasp constraint}
J_h \bmd{q} = G^T \bmd{x}_o
\end{equation}

The following assumptions are made for the grasp:

\begin{assumption}\label{asm:square Jh}
The multi-fingered hand has $m \geq 3 n$ joints.
\end{assumption}

\begin{assumption}\label{asm:full rank G}
The given multi-fingered grasp has $n>2$ contact points, which are non-collinear.
\end{assumption}

\begin{assumption}\label{asm:smooth surfaces}
The local fingertip and object contact surfaces are smooth.
\end{assumption}

\begin{remark} \label{rm:Jh to slip}
 Assumption \ref{asm:full rank G} ensures $G$ is always full rank \cite{Cole1989}, and can be ensured by a high-level grasp planner \cite{Hang2016}.
\end{remark}

\subsection{Hand-Contact Kinematics}\label{ssec:geometric parameters}

Here we review the differential geometric modeling of rolling contacts from \cite{Murray1994}. Note, the subscript $co$ will refer to the object surface of the contact, and the subscript $cf$ refers to the fingertip surface of the contact. At each contact point, we parameterize the contact surfaces by local coordinates  $\bm{\xi}_{co_i} = (a_{co_i}, b_{co_i}), \bm{\xi}_{cf_i} = (a_{cf_i}, b_{cf_i})$. The relation between the local coordinates and contact position vectors are defined by smooth mappings: $\bm{p}^{\mathcal{F}_i}_{fc_i} = \bm{c}_{cf_i}(\bm{\xi}_{cf_i}), \bm{p}^{\mathcal{O}}_{oc_i} = \bm{c}_{co_i}(\bm{\xi}_{co_i})$. The angle between $\frac{\partial \bm{c}_{co_i} }{\partial a_{co_i}}$ and  $\frac{\partial \bm{c}_{cf_i} }{\partial a_{cf_i}}$ is $\psi_i \in \mathbb{R}$.

The geometric parameters including the metric tensor, curvature tensor, and torsion tensor are used to define the rolling contact kinematics. For ease of notation, $\bm{c}_{fa}, \bm{c}_{fb}$ respectively denote $ \frac{\partial \bm{c}_{cf_i}}{\partial a_{cf_i}}$ and  $\frac{\partial \bm{c}_{cf_i}}{\partial b_{cf_i}}$. The Gauss frame is used to define the contact frame $\mathcal{C}_i$:

\begin{equation}\label{eq:gauss frame}
R_{fc_i} = \bracketmat{ccc}{\bm{\rho}_1 & \bm{\rho}_2 & \bm{\rho}_3} = \bracketmat{ccc}{  \frac{\bm{c}_{fa}}{|| \bm{c}_{fa} ||} & \frac{\bm{c}_{fa}}{|| \bm{c}_{fb} ||} & \bm{n}}
\end{equation}

where $R_{fc_i} \in SO(3)$ maps $\mathcal{C}_i$ to $\mathcal{F}_i$ and

\begin{equation}\label{eq:normal vector gauss frame}
\bm{n} = \frac{\bm{c}_{fa} \times \bm{c}_{fb}}{ || \bm{c}_{fa} \times \bm{c}_{fb} || }
\end{equation}

The metric tensor, $M_{cf_i} := M_{cf_i}(\bm{\xi}_{cf_i}) \in \mathbb{R}^{2\times 2}$, curvature tensor, $K_{cf_i} := K_{cf_i}(\bm{\xi}_{cf_i}) \in \mathbb{R}^{2\times 2}$, and torsion tensor, $T_{cf_i} := T_{cf_i}(\bm{\xi}_{cf_i})  \in \mathbb{R}^{2\times 1}$ are defined by:

\begin{equation}\label{eq:metric tensor}
M_{cf_i} = \bracketmat{cc}{ || \bm{c}_{fa} || & 0 \\ 0 & ||\bm{c}_{fb} || }
\end{equation}

\begin{equation}\label{eq:curvature tensor}
K _{cf_i}= \bracketmat{c}{\bm{\rho}_1^T \\ \bm{\rho}_2^T} \bracketmat{cc}{ \frac{\partial \bm{n}/\partial a_{cf_i}}{|| \bm{c}_{fa} ||} & \frac{\partial \bm{n}/\partial b_{cf_i}}{|| \bm{c}_{fb} ||}}
\end{equation}

\begin{equation}\label{eq:torsion tensor}
T_{cf_i} = \bm{\rho}_2^T \bracketmat{cc}{  \frac{\partial \bm{\rho}_1 / \partial a_{cf_i}}{|| \bm{c}_{fa} ||}  & \frac{\partial \bm{\rho}_1/\partial b_{cf_i}}{|| \bm{c}_{fb} ||}}
\end{equation}

Note the metric tensor $M_{co_i} := M_{co_i}(\bm{\xi}_{co_i}) \in \mathbb{R}^{2\times 2}$, curvature tensor, $K_{co_i} := K_{co_i}(\bm{\xi}_{co_i}) \in \mathbb{R}^{2\times 2}$, and torsion tensor, $T_{co_i} := T_{co_i}(\bm{\xi}_{co_i})  \in \mathbb{R}^{2\times 1}$ for the object can be defined by substituting $\bm{\xi}_{cf_i}$ with $\bm{\xi}_{co_i}$ in \eqref{eq:gauss frame}-\eqref{eq:torsion tensor}. Now we can define the dynamics governing the contact parameters $\bm{\xi}_{cf}$ and $\bm{\xi}_{co}$ by:

\begin{equation}\label{eq:contact kinematics}
\bmd{\xi}_{cf_i} = H_{cf_i} R_{c_i p} (\bm{\omega}_{f_i} - \bm{\omega}_o)
\end{equation}
\begin{equation}
\bmd{\xi}_{co_i} =H_{co_i} R_{c_i p}(\bm{\omega}_{f_i} - \bm{\omega}_o)
\end{equation}
where 
\begin{equation}\label{eq:contact derivative matrix H}
H_{cf_i}= M_{cf_i}^{-1} ( K_{cf_i} + R_{\psi_i} K_{co_i} R_{\psi_i} )^{-1} \bracketmat{ccc}{0 & -1 & 0 \\ 1 & 0 & 0},
\end{equation}
\begin{equation}
H_{co_i} =  M_{co_i}^{-1} R_{\psi_i}( K_{cf_i} + R_{\psi_i} K_{co_i} R_{\psi_i})^{-1}\bracketmat{ccc}{0 & -1 & 0 \\ 1 & 0 & 0}
\end{equation} 
 
\begin{equation}\label{eq:R_psi}
R_{\psi_i} = \bracketmat{cc}{ \cos(\psi_i) & -\sin(\psi_i) \\ -\sin(\psi_i) & -\cos(\psi_i)},
\end{equation}
\begin{equation}
\dot{\psi}_i = T_{cf_i} M_{cf_i} \bmd{\xi}_{cf_i} + T_{co_i} M_{co_i} \bmd{\xi}_{co_i}
\end{equation}

and $R_{c_i p} := R_{c_i p}(\bm{\xi}_{f_i}, \bm{q}_i ) \in SO(3)$ maps $\mathcal{P}$ to $\mathcal{C}_i$.

The chosen parameterizations  must satisfy \cite{Murray1994}:
\begin{assumption}\label{asm:geometric parameterization}
The parameterizations are orthogonal such that $\bm{c}_{fa}^T \bm{c}_{fb} = 0$, $\bm{c}_{oa}^T \bm{c}_{ob} = 0$, and $M_{cf_i}$, $K_{cf_i}$, $T_{cf_i}$, $M_{co_i}$, $K_{co_i}$, $T_{co_i}$ are defined for all $\bm{\xi}_{cf_i}$ on the fingertip surface, and $\bm{\xi}_{co_i}$ on the object surface, respectively.
\end{assumption}

\section{Zeroing Control Barrier Functions for Relative Degree Two Systems}\label{ssec:control barrier fcns}

Control barrier functions provide a formal method to ensure constraint satisfaction of dynamical systems. Zeroing control barrier functions is a subset of control barrier functions, which are known to be robust to modeling errors and conducive to real-time applications \cite{Ames2017}. Here we extend the existing work of \cite{Ames2017} to relative degree two systems for application in robotic grasping and manipulation.

Consider the following nonlinear affine system:
\begin{equation}\label{eq:nonlinear affine dynamics}
\bmdd{x} = \bm{f}(\bm{x}, \bmd{x}) + \bm{g}(\bm{x}, \bmd{x}) \bm{u}
\end{equation}
where $\bm{x} \in \mathbb{R}^p$ is the system state, $\bm{u} \in U \subseteq \mathbb{R}^m$ is the control input, and $\bm{f}, \bm{g}$ are locally Lipschitz continuous functions with respect to $\bm{x}$ and $\bmd{x}$. The goal of constraint satisfaction is to ensure the states $\bm{x}$ stay within a set of constraint-admissable states defined by:
\begin{equation}\label{eq:constraint set}
\mathcal{D} = \{ \bm{x} \in \mathbb{R}^p: h(\bm{x}) \geq 0 \}
\end{equation}
where $h: \mathbb{R}^p \to \mathbb{R} $ is a twice-continuously differentiable function of relative degree two.

\begin{definition} \label{def:extended class K}
A continuous function, $\alpha:(-b,a) \to (-\infty,\infty)$ for $a,b \in \mathbb{R}_{>0}$ is an \textit{extended class-$\mathcal{K}$ function} if it is strictly increasing and $\alpha(0) = 0$.
\end{definition} 

Let $B: \mathbb{R}^p \times \mathbb{R}^p \to \mathbb{R}$ be defined by:
\begin{equation}\label{eq:candidate zeroing cbf}
B(\bm{x}, \bmd{x}) = \frac{\partial h}{\partial \bm{x}}\bmd{x} + \alpha_1 (h(\bm{x})) 
\end{equation}
where $\alpha_1$ is an extended class-$\mathcal{K}$ function. Let the set $\mathcal{E}$ be defined by:
\begin{equation}\label{eq:set for zcbf}
\mathcal{E} = (\mathcal{D} \times \mathbb{R}^p) \bigcap \{ (\bm{x}, \bmd{x}) \in \mathbb{R}^p \times \mathbb{R}^p: B(\bm{x}, \bmd{x}) \geq 0 \}
\end{equation}

\begin{definition}\label{def:ZCBF}
Let $h:\mathbb{R}^p \to \mathbb{R}$ by a twice-continuously differentiable, relative degree two function for the system \eqref{eq:nonlinear affine dynamics}. A continuously differentiable function $B: \mathbb{R}^p \times \mathbb{R}^p \to \mathbb{R} $ defined by \eqref{eq:candidate zeroing cbf}, is a \textit{zeroing control barrier function} for the set $\mathcal{E}$ if there exists an extended class-$\mathcal{K}$ function $\alpha_2$ such that for all $(\bm{x}, \bmd{x}) \in  \mathcal{E}$,
\begin{equation}
\underset{\bm{u} \in U}{\text{sup}} \{ L_f B + L_g B \bm{u}  + \alpha_2(B) \} \geq 0 
\end{equation}
\end{definition}

Given a zeroing control barrier function $B$, for all $(\bm{x}, \bmd{x}) \in \mathcal{E}$ define the set:
\begin{equation}
S_u(\bm{x}, \bmd{x}) = \{ \bm{u} \in U: L_f B + L_g B \bm{u}  + \alpha_2(B) \geq 0 \}
\end{equation}

Constraint satisfaction using zeroing control barrier functions is ensured by the following theorem:

\begin{theorem}\label{thm:zcbf}
Given sets $\mathcal{D}, \mathcal{E}$ defined by \eqref{eq:constraint set}, \eqref{eq:set for zcbf} respectively, for a twice-continuously differentiable, relative degree two function $h$, if $(\bm{x}(0), \bmd{x}(0)) \in  \mathcal{E}$ and $B$ is a zeroing control barrier function, then for any Lipschitz control $\bm{u}:  \mathcal{E} \to U$ such that $\bm{u}(\bm{x}, \bmd{x}) \in S_u(\bm{x}, \bmd{x})$ for the system \eqref{eq:nonlinear affine dynamics}, $\bm{x}$ remains in $ \mathcal{D}$ for all $t > 0$.
\end{theorem}
\begin{proof}
For $(\bm{x}(0), \bmd{x}(0)) \in \mathcal{E}$, and $\bm{u} \in S_u$, it follows from Corollary 2 of \cite{Ames2017} that the set  $\{ (\bm{x}, \bmd{x}) \in \mathbb{R}^p \times \mathbb{R}^p: B(\bm{x}, \bmd{x}) \geq 0 \} $ is forward invariant. Thus $B \geq 0$ holds for a given extended class-$\mathcal{K}$ function $\alpha_1$. It follows from $B \geq 0$ that for all $\bm{x} \in \partial \mathcal{D}$, $\dot{h}(\bm{x}) \geq -\alpha_1(0) = 0$. By Nagumo's theorem \cite{Blanchini1999}, it follows that $\mathcal{D}$ is forward invariant. 
\end{proof}

\begin{remark}
The use of the two staggered extended class-$\mathcal{K}$ functions in Theorem \ref{thm:zcbf} allows a designer to adjust the performance of the control barrier functions for more aggressive/conservative constraint satisfaction, and takes into consideration bounds on the velocity to prevent large control effort near the constraint boundary. Note that for general nonlinear systems the condition that $(\bm{x}(0), \bmd{x}(0)) \in  \mathcal{E}$ may be restrictive, however in-hand manipulation tasks generally begin from a static grasp, such that satisfaction of $(\bm{x}(0), 0) \in \mathcal{E}$ is trivial.  However the condition does allow for more aggressive initial grasps such as if, for example, the hand catches the object prior to a manipulation task.
\end{remark}

\section{Proposed Solution}\label{sec:proposed solution}

In the following section, the set of constraint admissible grasp states is defined, and the zeroing control barrier function approach developed from Section \ref{ssec:control barrier fcns} is applied to derive the grasp constraints. The grasp constraints are then combined to define the proposed controller. 

\subsection{Constraint-Admissible Grasping States}

Satisfaction of the following condition ensures the object will not slip within the grasp \cite{ShawCortez2018b}:
\begin{equation}\label{eq:no slip condition}
\Lambda(\mu) R_{cp} \bm{f}_c >0 
\end{equation}
where $\mu \in \mathbb{R}_{>0}$ is the friction coefficient, $\Lambda(\mu) \in \mathbb{R}^{l_s \times 3 n}$ is a pyramid of $l_s \in \mathbb{Z}_{>0}$ faces used to a approximate the friction cone \cite{Nahon1992}, and $ R_{cp} \in \mathbb{R}^{3n \times 3n} $ is the block diagonal matrix of all $R_{c_i p}$ for $i \in [1,n]$. Let the set of constraint admissible contact forces be defined as:
\begin{equation}\label{eq:safe contact forces}
\mathcal{D}_{f} = \{ \bm{f}_c \in \mathbb{R}^{3n}: \Lambda(\mu) R_{cp} \bm{f}_c > 0 \}
\end{equation}

Let the set of admissible joint angles be defined by:
 \begin{equation}\label{eq:joint angle limits}
 \begin{split}
  h_{q\text{max}_j}(\bm{q}) &= - \bm{i}_j \bm{q} + q_{\text{max}_j} \geq 0, \forall j\in [1,m] \\
  h_{q\text{min}_j}(\bm{q}) &=  \bm{i}_j \bm{q} - q_{\text{min}_j}  \geq 0, \forall j\in [1,m]
 \end{split}
\end{equation} 
where $\bm{i}_j \in \mathbb{R}^{1\times m}$ is the $j$th row of $I_{m\times m}$ and $q_{\text{max}_j}, q_{\text{min}_j} \in \mathbb{R}_{\geq0}$ define the joint angle limits, which omit singular hand configurations. The set of constraint admissible joint angles is defined by:
\begin{equation}\label{eq:safe joint space}
\mathcal{D}_{q} = \{ \bm{q} \in \mathbb{R}^m: h_{q\text{max}}(\bm{q}) \geq 0, h_{q\text{min}} (\bm{q}) \geq 0 \} 
\end{equation}

The fingertip workspace is addressed here via the geometric modeling of the contact kinematics. Many existing tactile sensors are designed as flat or hemispherical fingertips \cite{Kappassov2015}, which can be appropriately modeled with geometric parameterizations \cite{Murray1994}. The benefit of the geometric modeling is not only that it can be applied to general fingertip shapes, but the fingertip workspace can be defined as simple box constraints: 
\begin{equation}\label{eq:fingertip workspace constraints}
\begin{split}
h_{r_1}(\bm{\xi}_{cf_i}) & = \bracketmat{cc}{1 & 0} \bm{\xi}_{cf_i} - a_{\text{min}} \\
h_{r_2}(\bm{\xi}_{cf_i}) & =  - \bracketmat{cc}{1 & 0} \bm{\xi}_{cf_i} + a_{\text{max}} \\ 
h_{r_3}(\bm{\xi}_{cf_i}) & =  \bracketmat{cc}{0 & 1} \bm{\xi}_{cf_i} - b_{\text{min}} \\
h_{r_4}(\bm{\xi}_{cf_i}) & =  - \bracketmat{cc}{0 & 1} \bm{\xi}_{cf_i} + b_{\text{max}} 
\end{split}
\end{equation}
where $a_{\text{min}}, a_{\text{max}}, b_{\text{min}}, b_{\text{max}} \in \mathbb{R}$ define the fingertip surface, and each $h_{r_j}$ define the box constraints such that if $h_{r_j} \geq 0, \forall j \in [1,4]$, the contact point is in the fingertip workspace. The set of allowable contact locations is $\mathcal{D}_r = \mathcal{D}_{r_1} \times ... \times \mathcal{D}_{r_n}$ for:
\begin{equation}\label{eq:safe conact locations}
\mathcal{D}_{r_i} = \{ \bm{\xi}_{cf_i} \in \mathbb{R}^2: h_{r_j}(\bm{\xi}_{cf_i}) \geq 0, \forall j \in [1,4] \}
\end{equation}

Let $\mathcal{H} =  \mathcal{D}_{f} \times \mathcal{D}_q \times \mathcal{D}_r$ denote the set of grasp constraint admissible states. In the set $\mathcal{H}$, the hand configuration is non-singular and thus by Assumption \ref{asm:square Jh} $J_h$ is full rank with rank $3n$. Furthermore, the contact points do not slip in $\mathcal{H}$ and so the grasp constraint \eqref{eq:grasp constraint} holds. Differentiation of \eqref{eq:grasp constraint}, and and substitution of \eqref{eq:hand dynamics} and \eqref{eq:object dynamics}  provides an expression for the contact forces as a function of the control torque, $\bm{u}$:

\begin{multline}\label{eq:contact force}
\bm{f}_c = B_{ho}^{-1}  \Big( J_h M_h^{-1} ( -C_h \bmd{q} + \bm{u} + \bm{\tau}_e) + \dot{J}_h \bmd{q} -\dot{G}^T  \bmd{x}_o\\
 + G^T M_o^{-1}( C_o \bmd{x}_o - \bm{w}_e) \Big)
\end{multline}
where $B_{ho} = ( J_h M_h^{-1} J_h^T + G^T M_o^{-1} G )$. Note that by Assumptions \ref{asm:square Jh} and \ref{asm:full rank G}, the inversion of $B_{ho}$ is well defined. 

\subsection{Proposed Controller}

The following constraints are developed to ensure forward invariance of $\mathcal{H}$ with respect to the states $(\bm{f}_c, \bm{q}, \bm{\xi}_f)$.

The no slip constraint is defined by substitution of \eqref{eq:contact force} into \eqref{eq:no slip condition} \cite{ShawCortez2018b}:
\begin{multline}\label{eq:no slip constraint}
\Lambda(\mu) R_{cp} B_{ho}^{-1} J_h M_h^{-1} \bm{u}   > \Lambda(\mu) R_{cp} B_{ho}^{-1} J_h M_h^{-1} (C_h \bmd{q} \\
- \bm{\tau}_e) - \dot{J}_h \bmd{q} + \dot{G}^T \bmd{x}_o - G^T M_o^{-1}( C_o \bmd{x}_o - \bm{w}_e) 
\end{multline}
Satisfaction of \eqref{eq:no slip constraint} by $\bm{u}$ directly ensures forward invariance of $ \mathcal{D}_{f}$. Let the set of admissible control torques for $\mathcal{D}_f $ be $S_{u_f} = \{ \bm{u} \in \mathbb{R}^m: \eqref{eq:no slip constraint} \text{ holds} \}$.

The zeroing control barrier functions from Section \ref{ssec:control barrier fcns} are used here to guarantee that the hand joints remain inside a desired joint space to prevent over-extension. Let the zeroing control barrier functions be defined by:
\begin{equation} \label{eq:joint angle zcbfs}
\begin{split}
B_{q\text{max}_j}(\bm{q}, \bmd{q}) &= \dot{h}_{q\text{max}_j} + \alpha_1(h_{q\text{max}_j}), \forall j \in [1,m] \\
B_{q\text{min}_j}(\bm{q}, \bmd{q}) &= \dot{h}_{q\text{min}_j} + \alpha_1(h_{q\text{min}_j}), \forall j \in [1,m] 
\end{split}
\end{equation}
where $\alpha_1(h)$ is an extended class-$\mathcal{K}$ function. Let $\mathcal{E}_q$ be defined by:
\begin{multline}
\mathcal{E}_q = (\mathcal{D}_q \times \mathbb{R}^m) \bigcap \{ (\bm{q}, \bmd{q}) \in \mathbb{R}^m \times \mathbb{R}^m: B_{q\text{max}_j} \geq 0, \\
B_{q\text{min}_j} \geq 0, \forall j \in [1,m] \} 
\end{multline}

Following Theorem \ref{thm:zcbf} and by using the dynamics \eqref{eq:hand dynamics} and contact force relation \eqref{eq:contact force}, if the control torque satisfies the following constraint for a given $\alpha_2(h)$ extended class-$\mathcal{K}$ function, then $\mathcal{D}_q$ is rendered forward invariant:
\begin{equation}\label{eq:joint position constraint}
A_q \bm{u} \geq \bm{b}_q
\end{equation} 
where $A_q \in \mathbb{R}^{2m \times m}$ and $ \bm{b}_q \in \mathbb{R}^{2m}$ are the respective concatenations of $L_g B_{q\text{max}_j}, L_g B_{q\text{min}_j}, j \in [1,m]$ and $-L_f B_{q\text{max}_j} - \alpha_2(B_{q\text{max}_j}), -L_f B_{q\text{min}_j} - \alpha_2(B_{q\text{min}_j}) \forall j \in [1,m]$. Let the set of admissable control torques for $\mathcal{D}_q$ be $S_{u_q} = \{ \bm{u} \in \mathbb{R}^m: A_q \bm{u} \geq \bm{b}_q \}$.

The zeroing control barrier functions are also used to ensure the contact points remain in the fingertip workspace. We define the candidate zeroing control barrier functions:
\begin{equation}\label{eq:control barrier function}
B_{r_j} ( \bm{\xi}_{f_i}, \bmd{\xi}_{cf_i})  = \dot{h}_{r_j}(\bm{\xi}_{cf_i}) + \alpha_1(h_{r_j}(\bm{\xi}_{cf_i})),  \ \forall j \in [1,4]
\end{equation}
with associated set $\mathcal{E}_r = \mathcal{E}_{r_1} \times ... \times \mathcal{E}_{r_n}$ where
\begin{multline}\label{eq:control barrier safe sets}
\mathcal{E}_{r_i} = ( \mathcal{D}_{r_i} \times \mathbb{R}^4) \bigcap \{ (\bm{\xi}_{cf_i}, \bmd{\xi}_{cf_i}) \in \mathbb{R}^{4}: \\
B_{r_j}( \bm{\xi}_{cf_i}, \bmd{\xi}_{cf_i})   \geq 0,  \ \forall j \in [1,4] \}
\end{multline}

Following Theorem \ref{thm:zcbf}, the following condition must be satisfied for forward invariance of $\mathcal{E}_r$:
\begin{equation}\label{eq:forward invariance condition}
L_f B_{r_j} + L_g B_{r_j}  \bm{u}_i   + \alpha_2 (B_{r_j}) \geq 0,\\
 \ \forall j \in [1,4], \forall i \in [1,n]
\end{equation}

By Assumption \ref{asm:smooth surfaces}, the terms $L_f B_{r_j}, L_g B_{r_j} $ are derived by differentiating \eqref{eq:contact kinematics} as follows:
\begin{multline}\label{eq:contact double derivative}
\bmdd{\xi}_{cf_i} = \frac{d}{dt} \Big[ H_{cf_i} R_{c_i p} \Big] (\bm{\omega}_{f_i} - \bm{\omega}_o)  + H_{cf_i} R_{c_i p}\bracketmat{cc}{0_{3\times 3} & I_{3\times 3}} \\
 \Big( \dot{J}_{s_i} \bmd{q}_i + J_{s_i} \bmdd{q}_i  - \bmdd{x}_o    \Big)
\end{multline}
Substitution of \eqref{eq:hand dynamics}, \eqref{eq:object dynamics}, and \eqref{eq:contact force} into \eqref{eq:contact double derivative}, provides the necessary derivation of $L_f B_{r_j}, L_g B_{r_j}$.

Re-writing \eqref{eq:forward invariance condition}, if the control torque satisfies the following constraint, then $\mathcal{D}_r$ is forward invariant:
\begin{equation}\label{eq:Linear u constraint}
A_r \bm{u} \geq \bm{b}_r 
\end{equation}
where $A_r \in \mathbb{R}^{4k \times m}$ and $\bm{b}_r \in \mathbb{R}^{4k}$ are  the respective column-wise concatenation over the $n$ contact points of:
\begin{equation}
A_{r_i} = \bracketmat{c}{L_g B_{r_1} (\bm{\xi}_{cf_i}, \bmd{\xi}_{cf_i}) \\  ... \\  L_g B_{r_4} (\bm{\xi}_{cf_i}, \bmd{\xi}_{cf_i}) }
\end{equation}
\begin{equation}
\bm{b}_{r_i} = \bracketmat{c}{-L_f B_{r_1} (\bm{\xi}_{cf_i}, \bmd{\xi}_{cf_i}) - \alpha_2(B_{r_1} (\bm{\xi}_{cf_i}, \bmd{\xi}_{cf_i})) \\ ... \\ -L_f B_{r_4} (\bm{\xi}_{cf_i}, \bmd{\xi}_{cf_i})- \alpha_2(B_{r_4} (\bm{\xi}_{cf_i}, \bmd{\xi}_{cf_i})) }
\end{equation}
Let the set of admissible control torques be $S_{u_r} = \{ \bm{u} \in \mathbb{R}^m: A_r \bm{u} \geq \bm{b}_r \}$.

Finally, to ensure the proposed controller is implementable on real systems the following actuator constraint is defined:
\begin{equation}\label{eq:actuator constraint}
\bm{u}_{\text{min}} \leq \bm{u}\leq  \bm{u}_{\text{max}}
\end{equation}
where $\bm{u}_{\text{min}}, \bm{u}_{\text{max}} \in \mathbb{R}^m$, and let $S_{u_\tau} = \{ \bm{u} \in \mathbb{R}^m: \bm{u}_{\text{min}} \leq \bm{u} \leq \bm{u}_{\text{max}} \}$.
Let the set of grasp constraint admissible control torques be $S_u = S_{u_f} \times S_{u_q} \times S_{u_r} \times S_{u_\tau}$. 
\begin{assumption}\label{asm:safe control torques nonempty}
The set of grasp constraint admissible control torques $S_u$ is non-empty.
\end{assumption}

The following proposed controller admits a nominal manipulation controller, $\bm{u}_{\text{nom}} \in \mathbb{R}^m$, and outputs a control torque that stays minimally close to $\bm{u}_{\text{nom}}$, in the 2-norm sense,  while adhering to the grasp constraints. The proposed control law is:
\begin{align} \label{eq:safe control qp}
\begin{split}
\bm{u}^* \hspace{0.1cm} = \hspace{0.1cm} & \underset{\bm{u}}{\text{argmin}}
\hspace{.1cm} (\bm{u} - \bm{u}_{\text{nom}})^T  (\bm{u} - \bm{u}_{\text{nom}}) \\
& \text{s.t.} 
\hspace{0.6cm} \eqref{eq:no slip constraint},  \eqref{eq:joint position constraint}, \eqref{eq:Linear u constraint}, \eqref{eq:actuator constraint}
\end{split}
\end{align}

\begin{proposition}\label{prop:Main}
Suppose Assumptions \ref{asm:square Jh}-\ref{asm:safe control torques nonempty} hold and a given $\bm{u}_{nom} \in \mathbb{R}^m$ is locally Lipschitz continuous. If $\bm{f}_c(0) \in \mathcal{D}_f$, $ (\bm{q}(0), \bmd{q}(0)) \in \mathcal{E}_q$, $ \bm{\xi}_f(0), \bmd{\xi}_f(0)) \in \mathcal{E}_r$, then \eqref{eq:safe control qp} applied to \eqref{eq:hand dynamics} ensures $\mathcal{H}$ is forward invariant.
\end{proposition}

\section{Results} \label{sec: results}

Numerical simulations compare a nominal manipulation controller from the literature with the same controller implemented via the proposed control \eqref{eq:safe control qp}. The controllers are implemented as if in a hierarchical grasp framework where a reference command is provided to the robotic hand from a high-level planner. The nominal controller is the computed torque control \cite{Murray1994, Cole1989} along with a conventional internal force control \cite{Ozawa2017} of the form: 
\begin{equation}\label{eq:nominal control}
\bm{u}_{\text{nom}} = J_h^T( G^\dagger \bm{u}_m + \bm{u}_f) - \bm{\tau}_e
\end{equation}
\begin{equation}\label{eq:nominal manipulation control}
\bm{u}_m = M_{ho} P ( \bmdd{r} +K_p \bm{e} + K_d \bmd{e} ) + C_{ho} \bmd{x}_o - \bm{w}_e
\end{equation}
\begin{equation} \label{eq:nominal internal force control}
\bm{u}_f = k_f (\bmb{p}_c - \bm{p}_{c_1},  \bmb{p}_c- \bm{p}_{c_2}, ...,  \bmb{p}_c - \bm{p}_{c_n} )
\end{equation}
where $\bm{e} \in \mathbb{R}^6$ is the error between the reference and object pose defined using Euler angles as in \cite{Cole1989}, $P \in \mathbb{R}^{6 \times 6}$ maps the object velocity using Euler angles to $\bmd{x}_o$ \cite{Cole1989},  $M_{ho} = M_o + G J_h^{-T} M_h J_h^{-1} G^T$, $C_{ho} = C_o + G J_h^{-T} ( C_h J_h^{-1} G^T + M_h \frac{d}{dt}[ J_h^{-1} G^T ] )$, $K_p, K_d \in \mathbb{R}^{6\times 6}$ are the respective proportional and derivative positive definite control gains, $\bm{r} \in \mathbb{R}^6$ is the reference command, $\bmb{p}_c =  \frac{1}{n}\sum_{i=1}^n \bm{p}_{c_i} $ is the centroid of the inertial contact positions, and $k_f \in \mathbb{R}_{>0}$ is a squeezing scalar gain. The gains chosen for the simulation are $K_p = 1.0* I_{6\times 6}$, $K_d = 2.5* I_{6\times 6}$, $k_f = 10.0$. 

\begin{figure}[h!]
\centering
	\subcaptionbox{Isometric view. \label{fig.initgraspiso}}
		{\includegraphics[scale=.185]{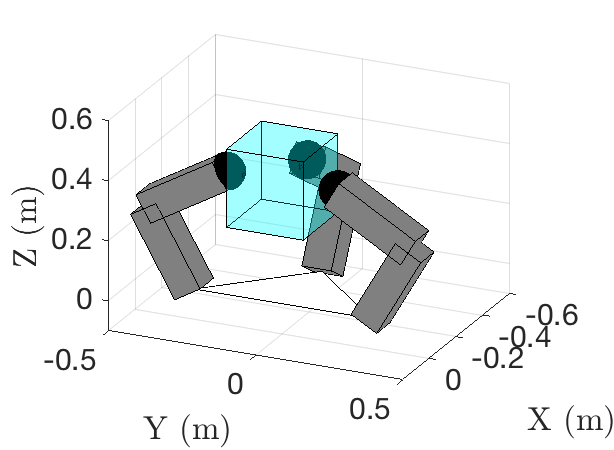}}
	\subcaptionbox{Top view. \label{fig.initgraspside} }
		{\includegraphics[scale=.185]{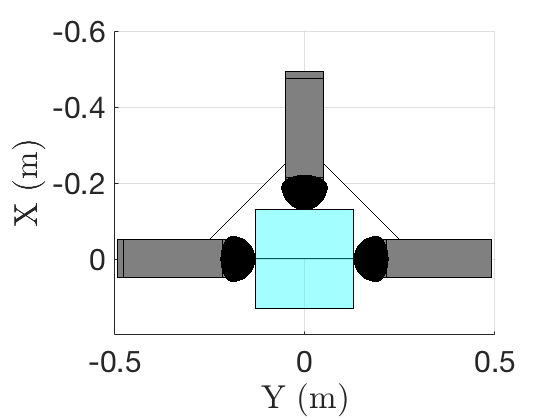}}
	\caption{Simulation setup.} \label{fig.initgrasp}
\end{figure}

\begin{figure}[h!]
\centering
	\subcaptionbox{Failed grasp configuration\label{fig.failgrasp_exp25}}
		{\includegraphics[scale=.21]{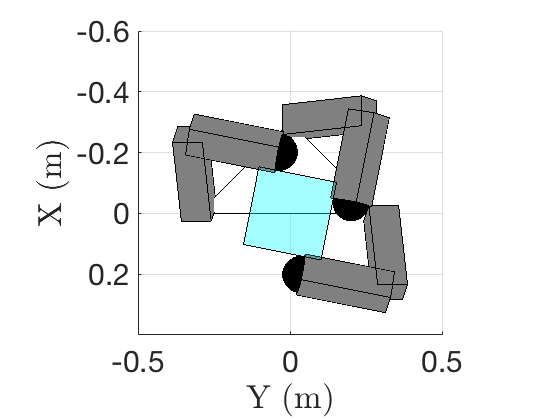}}
	\subcaptionbox{Contact points exceed fingertip workspace \label{fig.contact_traj_exp25_bfi} }
		{\includegraphics[scale=.21]{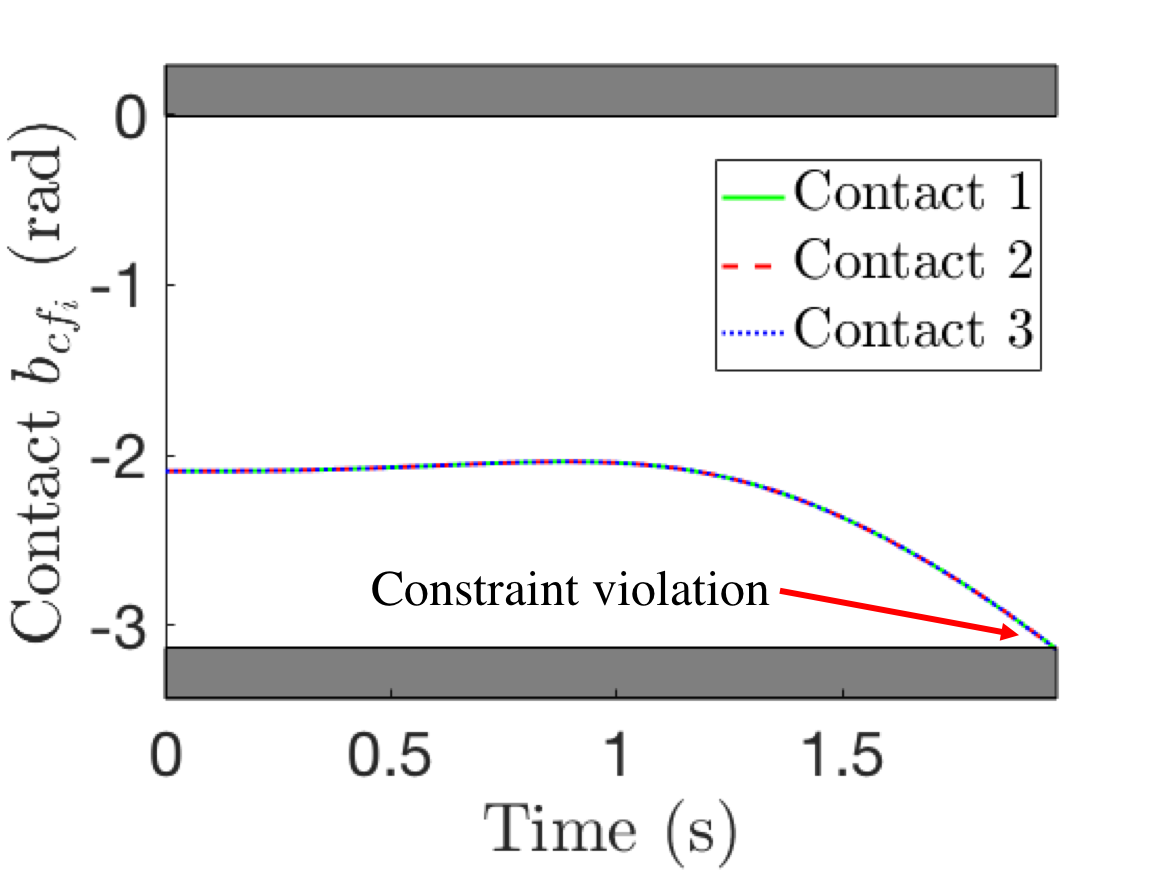}}
		\subcaptionbox{Joint limits exceeded \label{fig.joint_traj_exp25} }
		{\includegraphics[scale=.21]{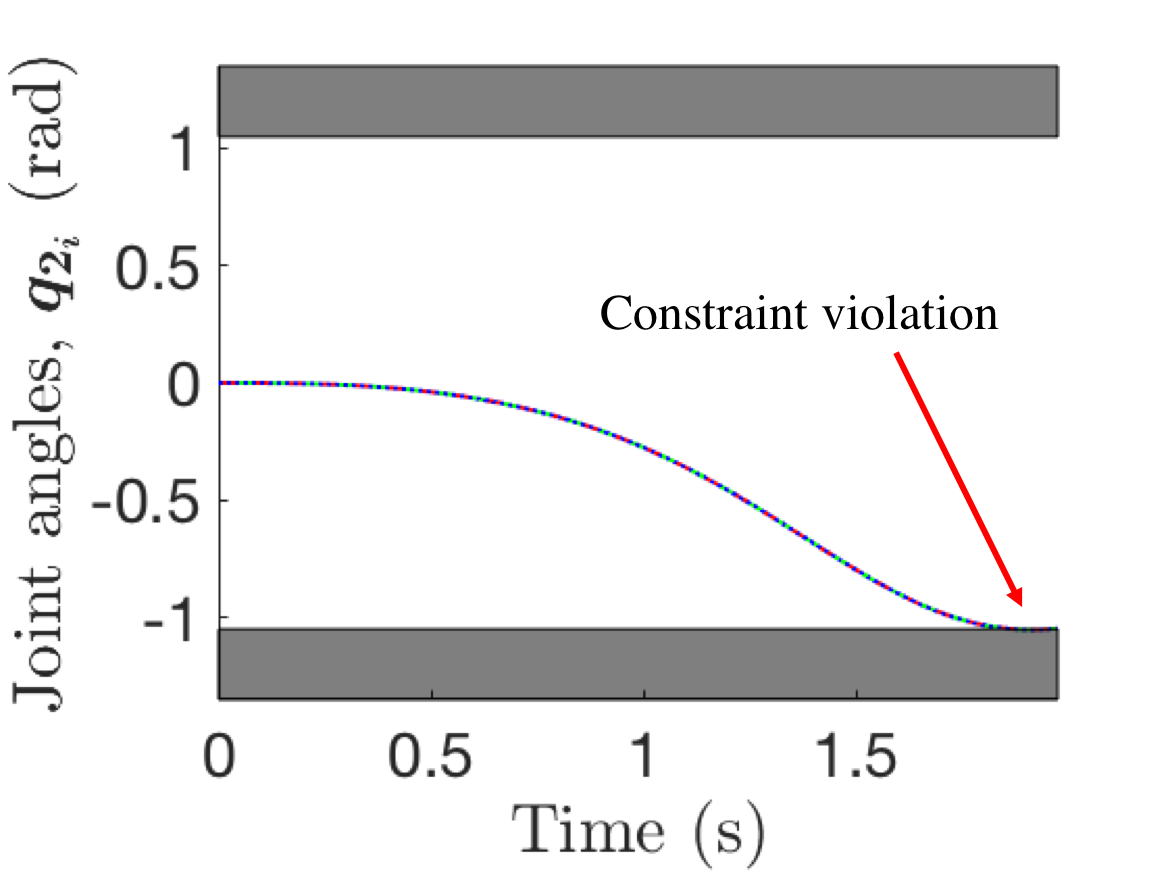}}
		\subcaptionbox{Slip constraint violated \label{fig.friction_exp25} }
		{\includegraphics[scale=.18]{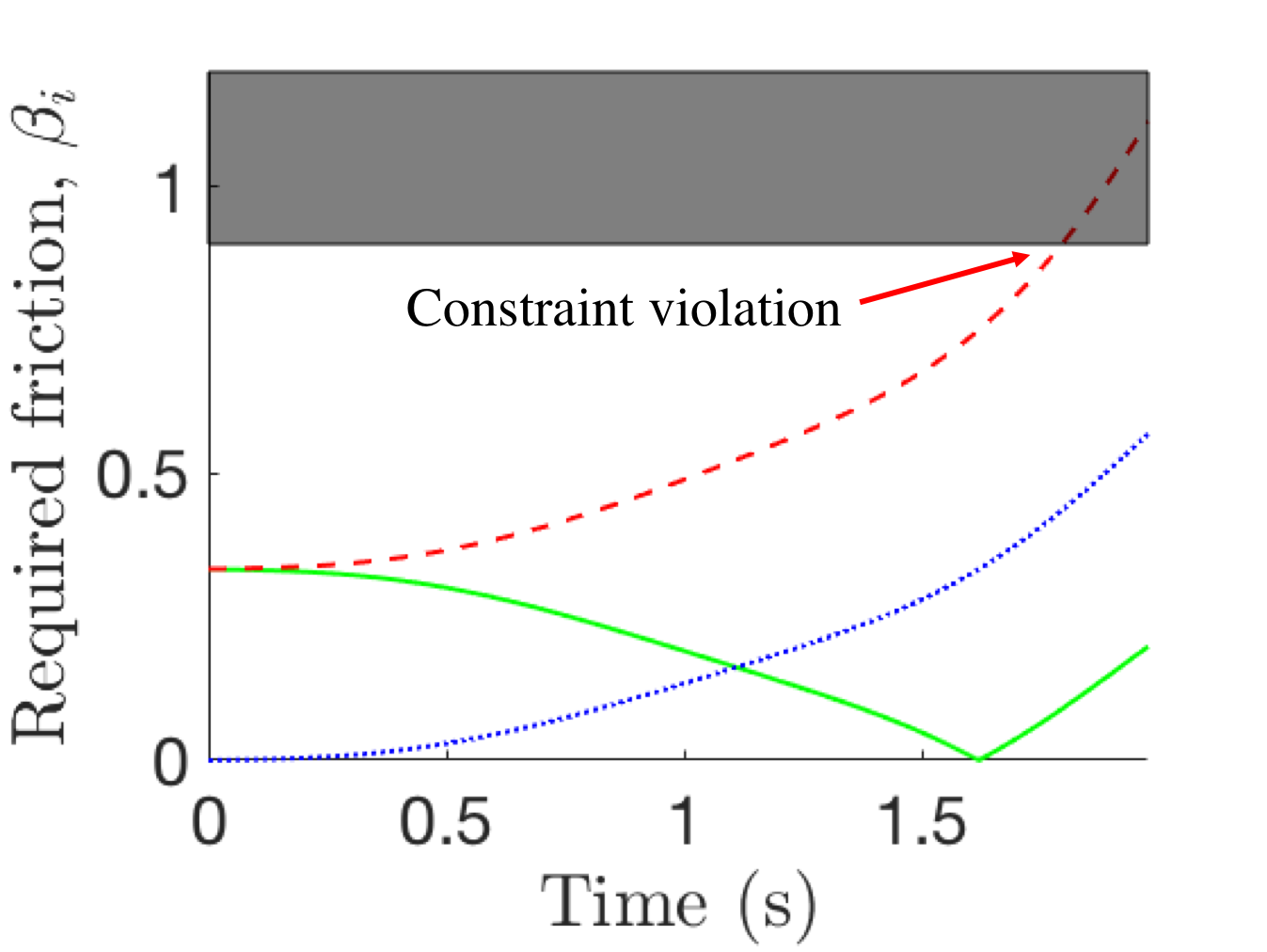}}
	\caption{Implementation of nominal control results in failed grasp due to constraint violations. } \label{fig.fail_grasp_exp25}
\end{figure}

\begin{figure}[h!]
\centering
	\subcaptionbox{\label{fig.contact_traj_exp26_afi}}
		{\includegraphics[scale=.21]{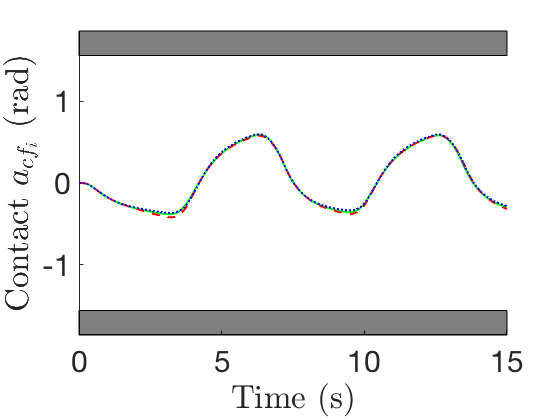}}
	\subcaptionbox{ \label{fig.contact_traj_exp26_bfi} }
		{\includegraphics[scale=.21]{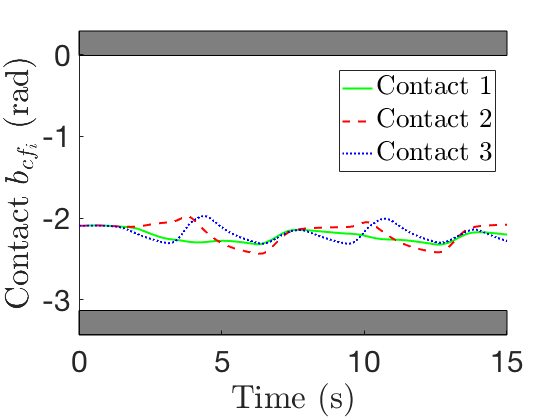}}
	\subcaptionbox{\label{fig.joint_traj_exp26_q1}}
		{\includegraphics[scale=.21]{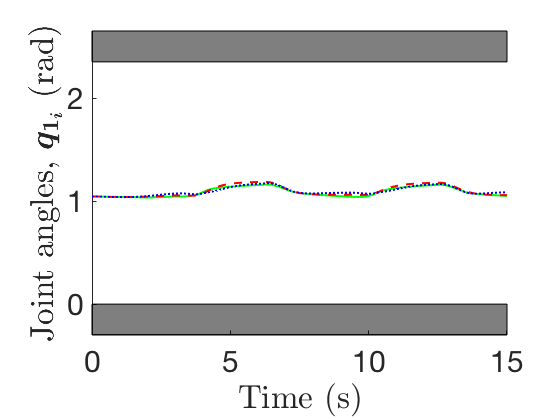}}
	\subcaptionbox{\label{fig.joint_traj_exp26_q2} }
		{\includegraphics[scale=.21]{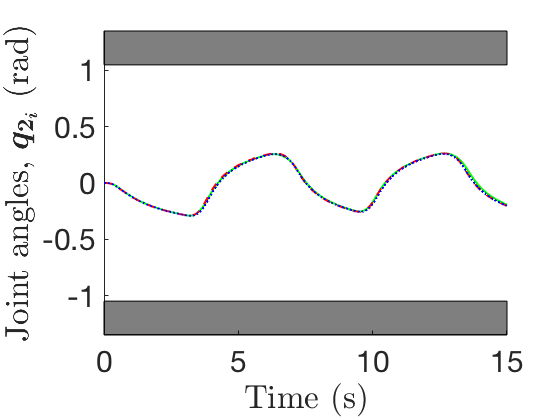}}
		\subcaptionbox{\label{fig.joint_traj_exp26_q3} }
		{\includegraphics[scale=.21]{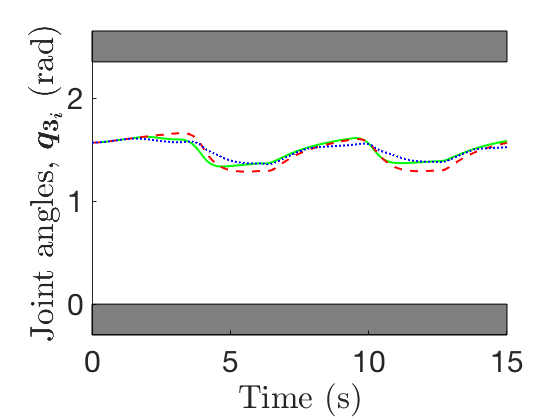}}
		\subcaptionbox{\label{fig.frictionexp26} }
		{\includegraphics[scale=.21]{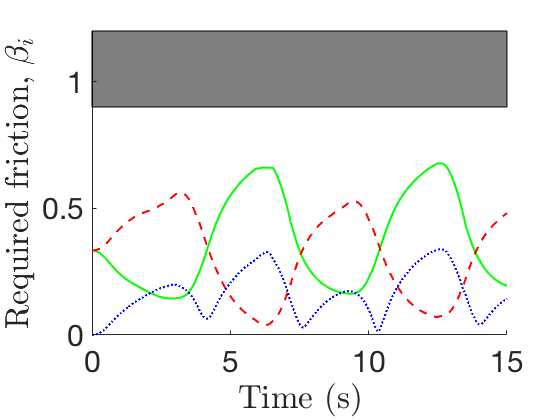}}
	\caption{Implementation of nominal + proposed control ensures grasp constraint satisfaction.} \label{fig.proposed_states}
\end{figure}

The reference $\bm{r} = (0,0,0.25 \cos(t),0,0,2 \cos(t))$ is provided to twist and pull the object about the $Z$-axis of the inertial frame, which is depicted in Figure \ref{fig.initgrasp} along with the initial static hand-object configuration. The grasped object is a $0.11$ kg cube with edge length of $0.2604$ m. The friction coefficient between the object and fingertip is $\mu = 0.9$. The hand is composed of identical rectangular prismatic links of dimension $ 0.3 \times 0.05 \times 0.05 \ \text{m}^3 $ with hemispherical fingertips of radius $R = 0.06$ m. Each finger consists of 2 revolute joints at the base and one between the two links.  The fingertip parameterization chosen to satisfy Assumption \ref{asm:geometric parameterization} is $\bm{c}_{cf_i} = [-R \cos(a_{cf_i}) \cos(b_{cf_i}), R \sin(a_{cf_i}), -R \cos(a_{cf_i})\sin(b_{cf_i})]^T$. The associated box constraints to define the fingertip workspace are: $-\pi/2 < a_{f_i} < \pi/2$, $-\pi < b_{f_i} < 0$. The joint angle limits for each finger are $\bm{q}_{\text{max}_i} = (3 \pi /2, \pi/3, 3 \pi/2), \bm{q}_{\text{min}_i} = (0, -\pi/3, 0), \ \forall i \in [1,3]$. The extended class-$\mathcal{K}$ functions used in the control barrier functions of \eqref{eq:safe control qp} are $\alpha_1(h) = \alpha_2(h) = h^3$. Note the only disturbance acting on the system is gravity. The simulations were performed using Matlab's ode3 integrator. The simulation time was 15 seconds, but simulations were stopped if the contact points exceed the fingertip workspace.

The implementation of the nominal control resulted in a failed grasp as shown in Figure \ref{fig.fail_grasp_exp25}. The gray regions denote areas outside of the constraint admissible state space. Figure \ref{fig.contact_traj_exp25_bfi} shows that as the nominal control rotates the object to track the reference pose, all of the contact trajectories, $b_{cf_i}$, exceed the fingertip workspace resulting in loss of contact. Figure \ref{fig.joint_traj_exp25} shows that the joint angles also exceed the prescribed joint limits. Figure \ref{fig.friction_exp25} shows the required friction, $\beta_i \in \mathbb{R}_{>0}$, which denotes the friction required at each contact to perform the manipulation motion \cite{ShawCortez2018b}. If the required friction $\beta_i$ exceeds the friction coefficient $\mu$, then the contact points slip, which is depicted in Figure \ref{fig.friction_exp25}. Thus the reference provided by the high-level grasp planner is infeasible for the nominal control. 

Figure \ref{fig.proposed_states} shows the grasp states from the proposed control \eqref{eq:safe control qp} with the nominal control \eqref{eq:nominal control}. Figures \ref{fig.contact_traj_exp26_afi} and \ref{fig.contact_traj_exp26_bfi} show the contact trajectories remain inside the fingertip workspace. Figures \ref{fig.joint_traj_exp26_q1}-\ref{fig.joint_traj_exp26_q3} show that the joint angles  remain within the joint angle limits. Figure \ref{fig.frictionexp26} shows that the required friction, $\beta_i$, remains below the slipping region. The combined satisfaction of all grasp constraints validates the forward invariance of the grasp constraint admissible set $\mathcal{H}$ as per Proposition \ref{prop:Main}. Note that the choice of $\alpha_1$ and $\alpha_2$ allows for more aggressive/conservative controller performance. The results show that the proposed controller prevents grasp failure even when the high-level planner provides an infeasible reference command for the given nominal control.

The numerical simulations show promising results for ensuring grasp constraint satisfaction during in-hand manipulation. However the proposed controller \eqref{eq:safe control qp} is only defined if Assumption \ref{asm:safe control torques nonempty} holds. Future work will investigate feasibility conditions of the proposed control.

\section{Conclusion} \label{sec: conclusion}

In this paper, a control law was proposed to guarantee grasp constraint satisfaction during in-hand manipulation. The grasp constraints were derived to ensure the object does not slip, the joints do not exceed joint angle limits, and the contact points do not leave the fingertip workspace. The proposed controller admits an existing manipulation controller, while adhering to the grasp constraints to support the assumptions made in the literature. Numerical results demonstrate the efficacy of the proposed approach.

\bibliographystyle{IEEEtran}
\bibliography{IEEEabrv,ShawCortez_CDC2018}

\end{document}